\documentclass[twoside,11pt]{article}

\usepackage{amsmath}
\usepackage{amsthm}
\usepackage{amsfonts}
\usepackage{fullpage}
\usepackage{hyperref}
\usepackage{amssymb}
\hypersetup{colorlinks,
            linkcolor=blue,
            citecolor=blue,
            urlcolor=magenta,
            linktocpage,
            plainpages=false}
\usepackage{url}
\usepackage{natbib}
\usepackage{algorithm2e}
\newcommand{\eps}{\epsilon}
\newcommand{\on}{\{-1,1\}}

\newcommand{\pr}{\mathop{\mathbf{Pr}}}
\newcommand{\E}{\mathop{\mathbf E}}
\newcommand{\R}{\mathbb{R}}
\newcommand{\D}{\mathcal{D}}
\newcommand{\ORR}{\mathsf{OR}}

\newcommand{\U}{\mathcal{U}}
\newcommand{\N}{\mathbb{N}}
\newcommand{\A}{\mathcal{A}}

\newcommand{\C}{\mathcal{C}}

\newcommand{\poly}{\mathsf{poly}}

\newcommand{\sign}{\mathsf{sign}}

\newcommand{\zo}{\{0, 1\}}
\newcommand{\zeto}[1]{\{0,\ldots,#1\}}
\renewcommand{\S}{\mathbb{S}}
\newcommand{\F}{\mathcal{F}}
\renewcommand{\P}{\mathcal{P}}
\newcommand{\eat}[1]{}
\newcommand{\cond}{\ |\ }

\newcommand{\alequn}[1]{\begin{align*} #1 \end{align*}}

\newtheorem{theorem}{Theorem}[section]
\newtheorem{lemma}[theorem]{Lemma}

\newtheorem{corollary}[theorem]{Corollary}

\newtheorem{fact}[theorem]{Fact}

\newtheorem{definition}[theorem]{Definition}
\newtheorem{problem}{Problem}

\title{Agnostic Learning of Disjunctions on Symmetric Distributions}
\author{Vitaly Feldman \\ IBM Research - Almaden
\and Pravesh Kothari\thanks{Work done while the author was at IBM Research - Almaden.}
\\ The University of Texas at Austin 
}

\begin{document}

\maketitle

\begin{abstract}
We consider the problem of approximating and learning disjunctions (or equivalently, conjunctions) on symmetric distributions over $\zo^n$. Symmetric distributions are distributions whose PDF is invariant under any permutation of the variables.
We prove that for every symmetric distribution $\D$, there exists a set of $n^{O(\log{(1/\epsilon)})}$ functions $\S$, such that for every disjunction $c$, there is function $p$, expressible as a linear combination of functions in $\S$, such that $p$ $\epsilon$-approximates $c$ in $\ell_1$ distance on $\D$ or $\E_{x \sim \D}[ |c(x)-p(x)|] \leq \epsilon$. This implies an agnostic learning algorithm for disjunctions on symmetric distributions that runs in time $n^{O( \log{(1/\epsilon)})}$. The best known previous bound is $n^{O(1/\epsilon^4)}$ and follows from approximation of the more general class of halfspaces \citep{Wim10}.
We also show that there exists a symmetric distribution $\D$, such that the minimum degree of a polynomial that $1/3$-approximates the disjunction of all $n$ variables in $\ell_1$ distance on $\D$ is $\Omega(\sqrt{n})$. Therefore the learning result above cannot be achieved via $\ell_1$-regression with a polynomial basis used in most other agnostic learning algorithms.

Our technique also gives a simple proof that for any product distribution $\D$ and every disjunction $c$, there exists a polynomial $p$ of degree $O(\log{(1/\epsilon)})$ such that $p$ $\epsilon$-approximates $c$ in $\ell_1$ distance on $\D$. This was first proved by \citet{BOW08} via a more involved argument.

\end{abstract}

\section{Introduction}
The goal of an agnostic learning algorithm for a concept class $\C$ is to produce, for any distribution on examples, a hypothesis $h$ whose error on a random example from the distribution is close to the best possible by a concept from $C$. This model reflects a common empirical approach to learning, where few or no assumptions are made on the process that generates the examples and a limited space of candidate hypothesis functions is searched in an attempt to find the best approximation to the given data.

Agnostic learning of disjunctions (or, equivalently, conjunctions) is a fundamental question in learning theory and a key step in learning algorithms for other concept classes such as DNF formulas and decision trees. Algorithms for this problem, such as the Set Covering Machine \citep{MarchandS02}, are also used in practical applications. There is no known efficient algorithm for the problem, in fact the fastest algorithm that does not make any distributional assumptions runs in $2^{\tilde{O}(\sqrt{n})}$ time \citep{KKMS05}. Polynomial-time learnability is only known when the examples are very close to being consistent with some disjunction \citep{AwasthiBS10}.

While the problem appears to be hard, strong hardness results are known only if the hypothesis is restricted to be a disjunction or a linear threshold function \citep{BenDavidEL:03,BshoutyBurroughs:06,FeldmanGKP09,FeldmanGRW:12}, or for learning using $\ell_1$-regression \citep{KS10}. Weaker, quasi-polynomial lower bounds are known assuming hardness of learning sparse parities with noise (see Section~\ref{sec:agnostic-learn}) and, very recently, hardness of refuting random SAT formulas \citep{DanielyS14}. It is also well-known that distribution-independent agnostic learning of disjunctions implies PAC learning of DNF expressions \citep{KearnsSS:94} (similar results for distribution specific-learning are discussed below). Finally, agnostic learning of disjunctions is known to be closely related to the problem of differentially-private release of answers to conjunctive queries \citep{GHRU11}.

We consider this problem with an additional assumption that example points are distributed according to a symmetric or a product distribution. Symmetric and product distributions are two incomparable classes of distributions that generalize the well-studied uniform distribution. Theoretical study of learning over symmetric distributions was first done by \citet{Wim10} who gave $n^{O(1/\epsilon^4)}$ time agnostic learning algorithm for the class of halfspaces.  Agnostic learning of disjunctions over symmetric distributions on $\zo^n$ also arises naturally in the well-studied problem of privately releasing answers to all short conjunction queries with low average error \citep{FeldmanKothari:14}.


\subsection{Our Results}
\label{sec:our-results}
We prove that disjunctions (and conjunctions) are learnable agnostically over any symmetric distribution in time $n^{O(\log(1/\eps))}$. This matches the well-known upper bound for the uniform distribution. Our proof is based on $\ell_1$-approximation of any disjunction by a linear combination of functions from a fixed set of functions. Such approximation directly gives an agnostic learning algorithm via $\ell_1$-regression based approach introduced by \citet{KKMS05}.

A natural and commonly used set of basis functions is the set of all monomials on $\zo^n$ of some bounded degree.  It is easy to see that on product distributions with constant bias, disjunctions longer than some constant multiple of $\log(1/\eps)$ are $\eps$-close to the constant function $1$. Therefore, polynomials of degree $O(\log(1/\eps))$ suffice for $\ell_1$ (or $\ell_2$) approximation on such distributions. This simple argument does not work for general product distributions. However it was shown by \citet{BOW08} that the same degree (up to a constant factor) still suffices in this case. Their argument is based on the analysis of noise sensitivity under product distributions and implies additional interesting results.

Interestingly, it turns out that low-degree polynomials cannot be used to obtain the same result for all symmetric distributions: we show that there exists a symmetric distribution for which disjunctions are no longer $\ell_1$-approximated by low-degree polynomials.
\begin{theorem}
\label{lem:hardsym-intro}
There exists a symmetric distribution $\D$ such that for $c = x_1 \vee x_2 \vee \cdots \vee x_n$, any polynomial $p$ that satisfies $\E_{x \sim \D}[ |c(x) -p(x)|] \leq 1/3$ is of degree $\Omega(\sqrt{n})$.
\end{theorem}
To prove this, we consider the standard linear program \citep[see][]{KS10} to find the coefficients of a degree $r$ polynomial that minimizes pointwise error with the disjunction $c$. The key idea is to observe that an optimal point for the dual can be used to obtain a distribution on which the \textit{$\ell_1$ error} of the best fitting polynomial $p$ for $c$ is same as the value of minimum \textit{pointwise error} of any degree $r$ polynomial with respect to $c$. When $c$ is a symmetric function, one can further observe that the distribution so obtained is in fact symmetric.  Combined with the degree lower bound for uniform approximation by polynomials by \citet{KS10}, we obtain the result. The details of the proof appear in Section~\ref{sec:poly-lower}.

Our approximation for general symmetric distributions is based on a proof that for the special case of the uniform distribution on $S_r$ (the points from $\on^n$ with Hamming weight $r$), low-degree polynomials still work, namely, for any disjunction $c$, there is a polynomial $p$ of degree at most $O(\log{(1/\epsilon)})$ such that the $\ell_1$ error $\E_{x \sim S_r}[ |c(x)-p(x)|] \leq \epsilon$. 
\begin{theorem}
For $r \in \{0,\ldots,n\}$, let $S_r$ denote the set of points in $\zo^n$ that have exactly $r$ $1$'s and let $\D_r$ denote the uniform distribution on $S_r$. For every disjunction $c$ and $\eps > 0$, there exists a polynomial $p$ of degree at most $O(\log{(1/\epsilon)})$ such that $\E_{\D_r}[|c(x) - p(x)|] \leq \eps$.\label{monotone-sym}
\end{theorem}
This result can be easily converted to a basis for approximating disjunctions over arbitrary symmetric distributions. All we need is to partition the domain $\zo^n$ into layers as $\cup_{0 \leq r \leq n}S_r$ and
use a (different) polynomial for each layer. Formally, the basis now contains functions of the form $\mathrm{IND}(r) \cdot \chi$, where $\mathrm{IND}$ is the indicator function of being in layer of Hamming weight $r$ and $\chi$ is a monomial of degree $O(\log(1/\eps))$. We note that a related strategy, of constructing a collection of functions, one for each layer of the cube was used by \citet{Wim10} to give $n^{O(1/\epsilon^4)}$ time agnostic learning algorithm for the class of halfspaces on symmetric distributions. However, his proof technique is based on an involved use of representation theory of the symmetric group and is not related to ours.

Our proof technique also gives a simpler proof for the result of \citet{BOW08} that implies approximation of disjunction by low-degree polynomials on all product distributions.
\begin{theorem}\label{th:product-poly-intro}
For any disjunction $c$ and product distribution $\D$ on $\zo^n$, there is a polynomial $p$ of degree $O(\log{(1/\epsilon)})$ such that $\E_{x \sim \D}[|c(x)-p(x)|] \leq \epsilon.$  \label{th:product-poly}
\end{theorem}

\subsection{Applications}
Theorem \ref{monotone-sym} together with a standard application of $\ell_1$ regression \citep{KKMS05} yields an agnostic learning algorithm for the class of disjunctions running in time $n^{O(\log(1/\eps))}$.
\begin{corollary}
There is an algorithm that agnostically learns the class of disjunctions on arbitrary symmetric distributions on $\zo^n$ in time $n^{O( \log{(1/\epsilon)})}$. \label{agnostic-learning}
\end{corollary}
This learning algorithm was extended to the class of all coverage functions in \citep{FeldmanKothari:14}, and then applied to the well-studied problem of privately releasing answers to all short conjunction queries with low average error.

It was shown by \citet{KalaiKM:09} and \citet{Feldman:10ab} that agnostic learning of conjunctions over a distribution $D$ in time $T(n,1/\eps)$ implies learning of DNF formulas with $s$ terms over $D$ in time $\poly(n,1/\eps) \cdot T(n,(4s/\eps))$. Further, under the same conditions distribution-specific agnostic boosting \citep{KalaiKanade:09,Feldman:10ab} implies that there exists an agnostic learning algorithm for decision trees with $s$ leaves running in time $\poly(n,1/\eps) \cdot T(n,s/\eps)$. Therefore we obtain quasi-polynomial learning algorithms for DNF formulas and decision trees over symmetric distributions.
\begin{corollary}
\begin{enumerate}
\item DNF formulas with $s$ terms are PAC learnable with error $\eps$ in time $n^{O(\log(s/\eps))}$ over all symmetric distributions;
\item Decision trees with $s$ leaves are agnostically learnable with excess error $\eps$ in time $n^{O(\log(s/\eps))}$ over all symmetric distributions.
\end{enumerate}
\end{corollary}

We also observe that any algorithm that agnostically learns the class of disjunction on the uniform distribution in time $n^{o(\log{(\frac{1}{\epsilon})})}$ would yield a faster algorithm for the notoriously hard problem of Learning Sparse Parities with Noise. This is implicit in prior work \citep{KKMS05,Feldman:12jcss} and we provide additional details in Section \ref{sec:agnostic-learn}.

\citet{DachmanFTWW:15} recently showed that $\ell_1$ approximation by polynomials is necessary and sufficient condition for agnostic learning over a product distribution (at least in the statistical query framework of \citet{Kearns:98}). Our agnostic learning algorithm (Theorem \ref{agnostic-learning}) and lower bound for polynomial approximation (Theorem \ref{lem:hardsym-intro}) demonstrate that this equivalence does not hold for non-product distributions.


\section{Preliminaries}
\label{sec:prelims}
We use $\zo^n$ to denote the $n$-dimensional Boolean hypercube. Let $[n]$ denote the set $\{1,2,\ldots,n\}$. For $S\subseteq [n]$, we denote by $\ORR_S: \zo^n \rightarrow \zo$, the monotone Boolean disjunction on variables with indices in $S$, that is, for any $x \in \zo^n$, $\ORR_S(x) = 0 \Leftrightarrow  \forall i \in S \ \ x_i = 0$.

One can define norms and errors with respect to any distribution $\D$ on $\zo^n$. Thus, for $f: \zo^n \rightarrow \R$, we write the $\ell_1$ and $\ell_2$ norms of $f$ as $\|f\|_1 = \E_{x \sim \D}[ |f(x)|] $ and $\|f\|_2 = \sqrt{\E[ f(x)^2]}$ respectively. The $\ell_1$ and $\ell_2$ error of $f$ with respect to $g$ are given by $\|f-g\|_1$ and $\|f-g\|_2$ respectively.

\subsection{Agnostic Learning}

The agnostic learning model is formally defined as follows \citep{Haussler:92,KearnsSS:94}.
\begin{definition}
\label{def:agnostic}
Let $\F$ be a class of Boolean functions and let $\D$ be any fixed distribution on $\zo^n$. For any distribution $\P$ over $\zo^n \times \zo$, let $\mbox{opt}(\P,\F)$ be defined as: $\mbox{opt}(\P,\F) =  \inf_{f \in \F} \E_{(x,y) \sim \P} [|y - f(x)|] .$ An algorithm $\A$, is said to agnostically learn $\F$ on $\D$ if for every {\em excess error} $\epsilon> 0$ and any distribution $\P$ on $\zo^n \times \zo$ such that the marginal of $\P$ on $\zo^n$ is $\D$, given access to random independent examples drawn from $\P$, with probability at least $\frac{2}{3}$, $\A$ outputs a hypothesis $h:\zo^n \rightarrow [0,1]$, such that $\E_{(x,y) \sim \P} [ |h(x)- y| ] \leq \mbox{opt}(\P, \F) + \epsilon.$
\end{definition}

It is easy to see that given a set of $t$ examples $\{(x^i,y^i)\}_{i\leq t}$ and a set of $m$ functions $\phi_1, \phi_2, \ldots, \phi_m$ finding coefficients $\alpha_1,\ldots,\alpha_m$ which minimize $$\sum_{i\leq t} \left| \sum_{ j \leq m} \alpha_j \phi_j (x^i) - y^i \right|$$ can be formulated as a linear program. This LP is referred to as Least-Absolute-Error (LAE) LP or Least-Absolute-Deviation LP, or $\ell_1$ linear regression. As observed by \citet{KKMS05}, $\ell_1$ linear regression gives a general technique for agnostic learning of Boolean functions.
\begin{theorem}
\label{th:lae-lp}
Let $\C$ be a class of Boolean functions, $\D$ be distribution on $\zo^n$ and $\phi_1, \phi_2, \ldots, \phi_m: \zo^n \rightarrow \R$ be a set of functions that can be evaluated in time polynomial in $n$. Assume that there exists $\Delta$
such that for each $f \in \C$, there exist reals $\alpha_1, \alpha_2, \ldots, \alpha_m$ such that $${\E_{x \sim \D} \left[  \left|\sum_{ i \leq m} \alpha_i \phi_i (x) - f(x)\right|\right] \leq \Delta}.$$
Then there is an algorithm that for every $\eps >0$ and any distribution $\P$ on $\zo^n \times \zo$ such that the marginal of $\P$ on $\zo^n$ is $\D$, given access to random independent examples drawn from $\P$, with probability at least $2/3$, outputs a function $h$ such that $$\E_{(x,y) \sim \P} [ |h(x)- y| ] \leq \Delta + \epsilon.$$ The algorithm uses $O(m/\eps^2)$ examples, runs in time polynomial in $n$, $m$, $1/\eps$ and returns a linear combination of $\phi_i$'s.
\end{theorem}
The output of this LP is not necessarily a Boolean function but can be converted to a Boolean function with disagreement error of $\Delta + 2\eps$ using $``h(x) \geq \theta"$ function as a hypothesis for an appropriately chosen $\theta$ \citep{KKMS05}.
\newcommand{\q}{\mathbf{q}}
\section{$\ell_1$ Approximation on Symmetric Distributions} \label{sec:l1approx}
In this section, we show how to approximate the class of all disjunctions on any symmetric distribution by a linear combination of a small set of basis functions.

As discussed above, polynomials of degree $O(\log{(1/\epsilon)})$ can $\eps$-approximate any disjunction in $\ell_1$ distance on any product distribution. This is equivalent to using low-degree monomials as basis functions. We first show that this basis would not suffice for approximating disjunctions on symmetric distributions. Indeed, we construct a symmetric distribution on $\zo^n$, on which, any polynomial that approximates the monotone disjunction $c = x_1 \vee x_2 \vee \ldots \vee x_n$ within $\ell_1$ error of $1/3$ must be of degree $\Omega(\sqrt{n})$.
\subsection{Lower Bound on $\ell_1$ Approximation by Low-Degree Polynomials}
\label{sec:poly-lower}
In this section we give the proof of Theorem \ref{lem:hardsym-intro}.
\begin{proof}[Proof of Theorem \ref{lem:hardsym-intro}]
Let $d:[n] \rightarrow \zo$ be the predicate corresponding to the disjunction $x_1 \vee x_2 \vee \ldots \vee x_n$, that is, $d(0) = 0$ and $d(i) = 1$ for each $i > 0$.

 Consider a natural linear program to find a univariate polynomial $f$ of degree at most $d$ such that $\|d-f\|_{\infty} = \max_{0 \leq i \leq n} |d(i) - f(i)| $ is minimized. This program (and its dual) often comes up in proving polynomial degree lower bounds for various function classes \citep[for example,][]{KS10}.

\begin{align*}
& \min  \text{   } \epsilon & \text{   }\\
s.t. \text{   } &\epsilon \geq |d(m) - \sum_{i = 0} ^{r} \alpha_i \cdot m^i|& \forall\text{  } m \in \zeto{n}\\
&\alpha_i \in \R \text{   } &\forall \text{  } i\in \zeto{r}
\end{align*}

If $\{\alpha_0, \alpha_1, \ldots, \alpha_n\}$ is a solution for the program above that has value $\eps$ then $f(m) = \sum_{i = 0}^r \alpha_i m^i$ is a degree $r$ polynomial that approximates $d$ within an error of at most $\epsilon$ at every point in $\zeto{n}$. \citet{KS10} show that there exists an $r^* = \Theta(\sqrt{n})$, such that the optimal value of the program above for $r = r^*$ is $\epsilon^* \geq 1/3$. Standard manipulations \citep[see][]{KS10} can be used to produce the dual of the program.
\begin{align*}
& \max  \text{   }\sum_{m=0}^n \beta_m \cdot d(m) &\\
s.t. \text{  } &\sum_{m=0}^n \beta_m \cdot m^i = 0 \text{   }& \forall \text{  } i\in \zeto{r}\\
 &\sum_{m=0}^n |\beta_m| \leq 1&\\
& \beta_m \in \R \text{  }& \forall \text{   } m \in \zeto{n}\\
\end{align*}
Let $\beta^* = \{ \beta^*_m\}_{m \in \zeto{n}} $ denote an optimal solution for the dual program with $r = r^*$. Then, by strong duality, the value of the dual is also $\epsilon^*$. Observe that $\sum_{m=0}^n |\beta^*_m|  = 1$, since otherwise we can scale up all the $\beta^*_m$ by the same factor and increase the value of the program while still satisfying the constraints.

Let $\rho:\zeto{n} \rightarrow [0,1]$ be defined by $\rho(m) = |\beta^*_m|$. Then $\rho$ can be viewed as a density function of a distribution on $\zeto{n}$ and we use it to define a symmetric distribution $\D$ on $\on^n$ as follows: $\D(x) = \rho(w(x))/{n \choose w(x)}$, where $w(x) = \sum_{i=1}^{n} x_i$ is the Hamming weight of point $x$. We now show that any polynomial $p$ of degree $r^*$ satisfies $\E_{x \sim \D}[ |c(x) - p(x)|] \geq 1/3$.

We now extract a univariate polynomial $f_p$ that approximates $d$ on the distribution with the density function $\rho$ using $p$. Let $p_{avg}:\on^n \rightarrow \R$ be obtained by averaging $p$ over every layer. That is, $p_{avg}(x) = \E_{z \sim \D_{w(x)}}[ p(z)]$, where $w(x)$ denotes the Hamming weight of $x$. It is easy to check that since $c$ is symmetric, $p_{avg}$ is at least as close to $c$ as $p$ in $\ell_1$ distance.

Further, $p_{avg}$ is a symmetric function computed by a multivariate polynomial of degree at most $r^*$ on $\zo^n$. Thus, the function $f_p(m)$ that gives the value of $p_{avg}$ on points of Hamming weight $m$ can be computed by a univariate polynomial of degree $r^*$. Further, $$\E_{x\sim \D}[|c(x)-p(x)|] \geq \E_{x\sim \D}[|c(x)-p_{avg}(x)|] = \E_{m \sim \rho}[ |d(m) - f_p(m)|].$$

Let us now estimate the error of $f_p$ w.r.t $d$ on the distribution $\rho$. Using the fact that $f_p$ is of degree at most $r^*$ and thus $\sum_{m = 0}^n f_p(m) \cdot \beta_m = 0$ (enforced by the dual constraints), we have:
\begin{align*}
\E_{m \sim \rho}[ |d(m) - f_p(m)|] &\geq \E_{m \sim \rho}[ (d(m) - f_p(m)) \cdot \sign ( \beta^*_m)]\\
&= \sum_{m = 0}^n d(m) \cdot \beta^*_m - \sum_{m = 0}^n f_p(m) \cdot \beta^*_m\\
& = \epsilon^* - 0 = \epsilon^* \geq 1/3.
\end{align*}
Thus, the degree of any polynomial that approximates $c$ on the distribution $\D$ with error of at most $1/3$  is $\Omega(\sqrt{n})$.
\end{proof}

\subsection{Upper Bound}
\label{sec:symmetric}

In this section, we describe how to approximate disjunctions on any symmetric distribution by using a linear combination of functions from a set of small size. Recall that $S_r$ denotes the set of all points from $\zo^n$ with weight $r$.

As we have seen above, symmetric distributions can behave very differently when compared to (constant bounded) product distributions. However, for the special case of the uniform distribution on $S_r$, denoted by $\D_r$, we show that for every disjunction $c$, there is a polynomial of degree $O(\log{(1/\epsilon)})$ that $\eps$-approximates it in $\ell_1$ distance on $\D_r$. As described in Section \ref{sec:our-results}, one can stitch together polynomial approximations on each $S_r$ to build a set of basis functions $\S$ such that every disjunction is well approximated by some linear combination of functions in $\S$. Thus, our goal is now reduced to constructing approximating polynomials on $\D_r$.



\eat{
\begin{lemma}
For any $0 \leq r \leq n$ and any disjunction $c$ there is a polynomial $p$ of degree at most $O( \log{(1/\epsilon)})$ such that
$\E_{\D_r}[|c(x) - p(x)|] \leq \eps$.
\label{monotone-sym}
\end{lemma}}
\begin{proof}[Proof of Theorem \ref{monotone-sym}]
We first assume that $c$ is monotone and without loss of generality $c=x_1 \vee \cdots \vee x_k$. We will also prove a slightly stronger claim that $\E_{\D_r}[|c(x) - p(x)|] \leq \E_{\D_r}[(c(x) - p(x))^2] \leq \eps$ in this case. Let $d:\zeto{k} \rightarrow \zo$ be the predicate associated with the disjunction, that is $d(i)=1$ whenever $i\geq 1$. Note that $c(x) = d\left(\sum_{i\in [k]} x_i\right)$. Therefore our goal is to find a univariate polynomial $f$ that approximates $d$ and then substitute $p_f(x) = f\left(\sum_{i\in [k]}x_i\right)$. This substitution preserves the total degree of the polynomial.
 We break our construction into several cases based on the relative magnitudes of $r,k$ and $\epsilon$.

If $k \leq 2 \ln{(1/\epsilon)}$, then the univariate polynomial that exactly computes the predicate $d$ satisfies the requirements. Thus assume that $k > 2 \ln(1/\epsilon)$.  If $r > n-k$, then, $c$ always takes the value $1$ on $S_r$ and thus the constant polynomial $1$ achieves zero error. If on the other hand, if $r \geq (n/k)\ln{(1/\epsilon)}$, then, $$\pr_{x \sim \D_r}[ c(x) = 0] = \frac{{{n-k} \choose r}}{{n \choose r}} = \prod_{i = 0}^{r-1} \left(1-\frac{k}{n-i}\right) \leq (1-k/n)^r \leq e^{-kr/n} \leq \epsilon.$$ In this case, the constant polynomial $1$ achieves an $\ell_2^2$ error of at most $\pr_{x \sim \D_r}[ c(x) = 0] \cdot 1 \leq \epsilon$.
Finally, observe that $r \leq (n/k) \ln{(1/\epsilon)}$ and $k > 2 \ln(1/\epsilon)$ implies $r \leq n/2$.  Thus, for the remaining part of the proof, assume that $r < \min \{n-k, (n/k) \ln{(1/\epsilon), n/2}\}$.

Consider the univariate polynomial $f:\zeto{k} \rightarrow \R$ of degree $t$ (for some $t$ to be chosen later) that computes the predicate $d$ exactly on $\zeto{t}$. This polynomial is given by $$f(w) = 1 - \frac{1}{t!} \prod_{i=1}^t (i-w)  =  \left\{ \begin{subarray}[
							1 - {w \choose t} \text{ for $w > t$}\\
                                                                         1                  \text{ for } 0 < w \leq t\\
                                                                          0                  \text{ for } w = 0\\
                                                                          \end{subarray}\right. $$

Let $$\delta_j = \pr_{x \sim \D_r}[ |\{ i \cond x_i = 1\}| = j] = \frac{ {{n-k} \choose {r-j} } \cdot {k \choose j} }{{n \choose r}}.$$ The $\ell_2^2$ error of $p_f(x)$ on $c$ satisfies,  $$||p_f-c||_2^2 = \E_{x \sim \D_r}[(c(x) - p_f(x))^2] = \sum_{j = t+1}^k \delta_j \cdot { {j \choose t}}^2 .$$ We denote the RHS of this equality by $\|d-f\|^2_2$.

We first upper bound $\delta_j$ as follows:
\begin{align*}\delta_j = \frac{{{n-k} \choose {r-j} } \cdot {k \choose j} }{{n \choose r} } &= \frac{(n-k)!}{(n-k-r+j)!(r-j)!} \cdot \frac{k!}{(k-j)!j!} \cdot \frac{(n-r)! r!}{n!}\\
&= \frac{1}{j!} \cdot \frac{r!}{(r-j)!} \cdot \frac{k!}{(k-j)!} \cdot \frac{(n-r)!}{n!} \cdot \frac{(n-k)!}{(n-k-r+j)!}\\
&\leq \frac{1}{j!} \cdot (rk)^j \cdot \frac{(n-k) \cdot (n-k-1) \cdots (n-k-r+j+1)}{n \cdot (n-1) \cdots (n-r+1)}\\
&\leq \frac{1}{j!} \cdot (n \ln{(1/\epsilon)})^j \cdot \frac{1}{(n-r+j) \cdot (n-r+j-1) \cdots (n-r+1)},
\end{align*}
where, in the second to last inequality, we used that $r <n/k \ln {(1/\epsilon)}$ to conclude that $rk \leq (n \ln{(1/\epsilon)})$.
Now, $r < n/2$ and thus $(n-r+1) > n/2$. Therefore,
\begin{align*}
\delta_j &\leq \frac{2^j \cdot (n\ln{(1/\epsilon)})^j}{n^j \cdot j!} = \frac{(2\ln{(1/\epsilon)})^j}{j!},
\end{align*}

and thus: $$\|d-f\|^2_2\leq  \sum_{j = t+1}^k {j \choose t}^2 \frac{(2 \ln{(1/\epsilon)})^j}{j!}.$$

Set $t = 8e^2 \ln{(1/\epsilon)} $. Using $j! > (j/e)^j > (t/e)^j$ for every $j \geq t+1$, we obtain:

\begin{align}
\|d-f\|_2^2 &\leq \sum_{j = t+1}^k 2^{2j}  \cdot \left( \frac{2 \ln{(1/\epsilon)}}{8e \ln{(1/\epsilon)}} \right)^j \leq \epsilon \cdot \sum_{j = t+1}^{\infty} 1/e^j \leq \epsilon. \label{eq:final-bound}
\end{align}

To see that $\E_{\D_r}[|c(x) - p(x)|] \leq \E_{\D_r}[(c(x) - p(x))^2]$ we note that in all cases and for all $x$, $|p(x)-c(x)|$ is either $0$ or $\geq 1$. This completes the proof of the monotone case.

We next consider the more general case when $c = x_1 \vee x_2 \vee \ldots \vee x_{k_1} \vee \bar{x}_{k_1+1} \vee  \bar{x}_{k_1+2} \vee \ldots \vee  \bar{x}_{k_1+k_2}$. Let $c_1 = x_1 \vee x_2 \vee \ldots \vee x_{k_1}$ and $c_2 =  \bar{x}_{k_1+1} \vee  \bar{x}_{k_1+2} \vee \ldots \vee  \bar{x}_{k_1+k_2}$ and $k = k_1 + k_2$. Observe that $c = 1 - (1-c_1) \cdot (1-c_2) = c_1+c_2-c_1c_2$.

Let $p_1$ be a polynomial of degree $O(\log{(1/\epsilon)})$ such that $\|c_1 - p_1\|_1 \leq \|c_1 - p_1\|_2^2 \leq \epsilon/3$. Note that if we swap 0 and 1 in $\zo^n$ then $c_2$ will be equal to a monotone disjunction $\bar{c}_2 = x_{k_1+1} \vee  x_{k_1+2} \vee \ldots \vee  x_{k_1+k_2}$ and $\D_r$ will become $\D_{n-r}$. Therefore by the argument for the monotone case, there exists a polynomial $\bar{p}_2$ of degree $O(\log{(1/\epsilon)})$ such that $\|\bar{c}_2 - \bar{p}_2\|_1 \leq \epsilon/3$. By renaming the variables back we will obtain a polynomial $p_2$ of degree $O(\log{(1/\epsilon)})$ such that $\|c_2 - p_2\|_1 \leq \|c_2 - p_2\|_2^2 \leq \epsilon/3$.
Now let $p  = p_1+p_2 - p_1p_2$. Clearly the degree of $p$ is $O( \log{(1/\epsilon)})$. We now show that $\|c-p\|_1\leq \epsilon$:
\alequn{
\E_{x\sim \D_r}[|c(x)-p(x)|] &= \E_{x\sim \D_r}[|(1-c(x))-(1-p(x))|] \\
& = \E_{x\sim \D_r}[|(1-c_1)(1-c_2)-(1-p_1)(1-p_2)|]\\
&= \E_{x\sim \D_r}[|(1-c_1)(p_2 -c_2) + (1-c_2)(p_1-c_1) - (c_1-p_1)(c_2-p_2)|] \\
&\leq \E_{x\sim \D_r}[|(1-c_1)(p_2 -c_2)|] + \E_{x\sim \D_r}[|(1-c_2)(p_1-c_1)|] +\E_{x\sim \D_r}[|(c_1-p_1)(c_2-p_2)|] \\
&\leq \E_{x\sim \D_r}[|p_2 -c_2|] + \E_{x\sim \D_r}[|p_1-c_1|] + \sqrt{\E_{x\sim \D_r}[(c_1-p_1)^2]\E_{x\sim \D_r}[(c_2-p_2)^2]} \\
& \leq \eps/3 +\eps/3+\eps/3 = \eps.
}
\end{proof}

\section{Polynomial Approximation on Product Distributions}
\label{sec:product}
In this section, we show that for every product distribution $\D = \prod_{i \in [n]} \D_i$, every $\epsilon > 0$ and every disjunction (or conjunction) $c$ of length $k$, there exists a polynomial $p: \zo^n \rightarrow \R$ of degree {$O(\log{(1/\epsilon)})$} such that $p$ $\eps$-approximates $c$ in $\ell_1$ distance on $\D$.

\eat{
\begin{lemma}
Let $\D = \prod_{i \in [n]} \D_i$ be any product distribution on $\zo^n$ with $\pr_{x_i \sim \D^i}[x_i = 1] = \mu_i $ for each $i \in [k]$ and let $c = x_1 \vee x_2 \vee \ldots \vee x_k$ be the monotone disjunction on all the variables. For any $ \epsilon > 0$, there exists a polynomial $f: \zo^k \rightarrow \R$ of degree $O( \log{(1/\epsilon)})$ such that $$\|c-f\|_1 = \E_{x \in \D} |c(x) - f(x)| \leq \epsilon.$$
\label{lem:product-poly}
\end{lemma}}

\begin{proof}[Proof of Theorem \ref{th:product-poly-intro}]
First, we note that without loss of generality we can assume that the disjunction $c$ is equal to $x_1 \vee x_2 \vee \cdots \vee x_k$ for some $k\in[n]$. We can assume monotonicity since we can convert negated variables to un-negated variables by swapping the roles of $0$ and $1$ for that variable. The obtained distribution will remain product after this operation. Further we can assume that $k=n$ since variables with indices $i > k$ do not affect probabilities of variables with indices $\leq k$ or the value of $c(x)$.

We first note that we can assume that $\pr_{x \sim \D} [x = 0^k] > \epsilon$ since, otherwise, the constant polynomial $1$ gives the desired approximation. Let $\mu_i = \pr_{x_i \sim \D^i}[x_i = 1]$. Since $c$ is a symmetric function, its value at any $x \in \zo^k$ depends only on the Hamming weight of $x$ that we denote by $w(x)$. Thus, we can equivalently work with the univariate predicate $d:\{0,1,\ldots,k\} \rightarrow \zo$, where $d(i) = 1$ for $i >0$ and $d(0) = 0$.

As in the proof of Theorem \ref{monotone-sym}, we will approximate $d$ by a univariate polynomial $f$ and then use the polynomial $p_f(x) = f(w(x))$ to approximate $c$. 

Let $f:\{0,1, \ldots, k\} \rightarrow \R$ be the univariate polynomial of degree $t$ that matches $d$ on all points in $\{0,1, \ldots, t\}$. Thus, $$f(w) = 1 - \frac{1}{t!} \cdot \prod_{i=1}^t (w-i) = \left\{ \begin{subarray}[
							1 - {w \choose t} \text{ for $w > t$}\\
                                                                          1                  \text{ for } 0 < w \leq t\\
                                                                          0                  \text{ for } w= 0\\
                                                                          \end{subarray}\right. $$

We have,
$$\E_{x \sim \D_r}[(c(x) - p_f(x))^2] = \sum_{j = 0}^{k} \pr_{x \sim \D}[ w(x) = j] \cdot |d(j)-f(j)|$$
and we denote the RHS of this equation by $\|d-f\|_1$.

Then:
\begin{align}
\|d-f\|_1 &= \sum_{j = t+1}^k \pr_{ \D} [ w(x) = j] \cdot |1-f(j)| \nonumber\\
&= \sum_{j = t+1}^k \pr_{ \D} [ w(x) = j] \cdot {j \choose t} . \label{error-exp}
\end{align}

Let us now estimate $\pr_{\D}[w(x) = j]$.
\begin{align*}
\pr_{\D}[ w(x)= j] & = \sum_{S \subseteq [n]\text{, } |S| = j} \prod_{i \in S} \mu_i \cdot \prod_{i \notin S} (1-\mu_i ) \\
&\leq \sum_{S \subseteq [n]\text{, } |S| = j} \prod_{i \in S} \mu_i \\
\end{align*}

Observe that in the expansion of $(\sum_{i=1}^k \mu_i)^j$, the term $\prod_{i \in S} \mu_i$ occurs exactly $j!$ times. Thus, $$ \sum_{S \subseteq [n]\text{, } |S| = j} \prod_{i \in S} \mu_i    \leq \frac{( \sum_{i=1}^k \mu_i)^j }{j!}.$$

Set $\mu_{avg} = \frac{1}{k} \sum_{i=1}^k \mu_i$. We have: $$\epsilon \leq \pr_{x \sim \D} [ x = 0^k] = \prod_{i=1}^k (1-\mu_i) \leq \left(1 - \frac{1}{k} \cdot \sum_{i=1}^k \mu_i\right)^k = (1-\mu_{avg})^k.$$

Thus, $\mu_{avg} = c/k$ for some $c \leq 2\ln{(1/\epsilon)}$ whenever $k \geq k_0$ where $k_0$ is some universal constant. In what follows, assume that $k \geq k_0$. (Otherwise, we can use the polynomial of degree equal to $k$ that exactly computes the predicate $d$ on all points).

We are now ready to upper bound the error $\|d-f\|_1$.  From Equation \eqref{error-exp}, we have:
\begin{align*}
\|d-f\|_1 &= \sum_{j = t+1}^k \pr_{ \D} [ w(x) = j] \cdot {j \choose t} \leq \sum_{j = t+1}^k \frac{( \sum_{i=1}^k \mu_i)^j }{j!} \cdot {j \choose t}\\
& \leq  \sum_{j = t+1}^k {j \choose t}  \cdot \frac{(2\ln(1/\epsilon) )^j }{j!}
\end{align*}
Setting $t = 4e^2 \ln {(1/\epsilon)}$ and using the calculation from Equation \eqref{eq:final-bound} in the proof of Thm.~\ref{monotone-sym}, we obtain that the error $\|d-f\|_1 \leq \epsilon$.
\end{proof}

\section{Agnostic Learning of Disjunctions}
\label{sec:agnostic-learn}
Combining Thm.~\ref{th:lae-lp} with the results of the previous section (and the discussion in Section \ref{sec:our-results}), we obtain an agnostic learning algorithm for the class of all disjunctions on product and symmetric distributions running in time $n^{O(\log{(1/\epsilon)})}$.

\begin{corollary}[Cor.~\ref{agnostic-learning}, restated]
There is an algorithm that agnostically learns the class of disjunctions on any product or symmetric distribution on $\zo^n$ with excess error of at most $\epsilon$ in time $n^{O( \log{(1/\epsilon)})}$.
\end{corollary}

We now remark that any algorithm that agnostically learns the class of disjunctions (or conjunctions) on $n$ inputs on the uniform distribution on $\zo^n$ in time $n^{o(\log{(\frac{1}{\epsilon})})}$ would yield a faster algorithm for the notoriously hard problem of Learning Sparse Parities with Noise(SLPN). The reduction is based on the technique implicit in the work of \citet{KKMS05} and \citet{Feldman:12jcss}.

For $S \subseteq [n]$, we use $\chi_S$ to denote the parity of inputs with indices in $S$. Let $\U$ denote the uniform distribution on $\zo^n$. We say that random examples of a Boolean function $f$ have noise of rate $\eta$ if the label of a random example equals $f(x)$ with probability $1 - \eta$ and $1-f(x)$ with probability $\eta$.
\begin{problem}[Learning Sparse Parities with Noise]
For $\eta \in (0,1/2)$ and $k \leq n$ the problem of learning $k$-sparse parities with noise $\eta$ is the problem of finding (with probability at least $2/3$) the set $S \subseteq [n]$,$|S| \leq k$, given access to random examples with noise of rate $\eta$ of parity function $\chi_S$.
\end{problem}
The fastest known algorithm for learning $k$-sparse parities with noise $\eta$ is a recent breakthrough result of Valiant \citeyearpar{Val12} which runs in time $O(n^{0.8k} \poly(\frac{1}{1-2\eta}))$ .

\citet{KKMS05} and \citet{Feldman:12jcss} prove hardness of agnostic learning of majorities and conjunctions, respectively, based on correlation of concepts in these classes with parities. We state below this general relationship between correlation with parities and reduction to SLPN, a simple proof of which appears in \citep{FeldmanKV:13}.
\begin{lemma}
Let $\C$ be a class of Boolean functions on $\zo^n$. Suppose, there exist $\gamma >0$ and $k \in \N$ such that for every $S \subseteq [n]$, $|S| \leq k$, there exists a function, $f_S \in \C$, such that $|\E_{x \sim \U} [f_S(x) \chi_S(x)]|\geq \gamma(k)$. If there exists an algorithm $\A$ that learns the class $\C$ agnostically with excess error $\epsilon$ in time $T(n, \frac{1}{\epsilon})$ then, there exists an algorithm $\A'$ that learns $k$-sparse parities with noise $\eta < 1/2$ in time $\poly(n,\frac{1}{(1-2\eta)\gamma(k)}) + 2 T(n, \frac{2}{(1-2\eta)\gamma(k)})$. \label{cor2lpn}
\end{lemma}

The correlation between a disjunction and a parity is easy to estimate.
\begin{fact}
For any $S \subseteq [n]$, $|\E_{x \sim \U} [\ORR_S(x) \chi_S(x)]| = \frac{1}{2^{|S|-1}}$.
\end{fact}
We thus immediately obtain the following corollary.
\begin{theorem}
Suppose there exists an algorithm that learns the class of Boolean disjunctions over the uniform distribution agnostically with excess error of $\epsilon > 0$ in time $T(n,\frac{1}{\epsilon})$. Then there exists an algorithm that learns $k$-sparse parities with noise $\eta < \frac{1}{2}$ in time $\poly(n,\frac{2^{k-1}}{1-2\eta})+ 2T(n, \frac{2^{k-1}}{1-2\eta})$. In particular, if $T(n, \frac{1}{\epsilon}) = n^{o(\log{(1/\epsilon)})}$, then, there exists an algorithm to solve $k$-SLPN in time $n^{o(k)}$.
\end{theorem}
Thus, any algorithm that is asymptotically faster than the one from Cor.~\ref{agnostic-learning} yields a faster algorithm for $k$-SLPN.

\bibliographystyle{plainnat}
\bibliography{references}
\appendix

\end{document}